\def\year{2022}\relax
\documentclass[letterpaper]{article} 
\usepackage{stylefile}  
\usepackage{times} 
\usepackage{helvet} 
\usepackage{courier} 
\usepackage[hyphens]{url} 
\usepackage{graphicx}
\usepackage{natbib}  
\usepackage{caption}
\DeclareCaptionStyle{ruled}{labelfont=normalfont,labelsep=colon,strut=off}
\usepackage{algorithm}
\usepackage{algorithmic}

\usepackage{newfloat}
\usepackage{listings}
\floatstyle{ruled}
\newfloat{listing}{tb}{lst}{}
\floatname{listing}{Listing}

\nocopyright

\title{Reinforcement Learning with Information-Theoretic Actuation}
\author {
   Elliot Catt\textsuperscript{\rm 1\footnote{Contact Author}}, 
    Marcus Hutter\textsuperscript{\rm 1,\rm 2}, 
    Joel Veness \textsuperscript{\rm 2}
}
\affiliations {
    \textsuperscript{\rm 1} Australian National University\\
    \textsuperscript{\rm 2} Deepmind\\
    elliot.carpentercatt@anu.edu.com, www.hutter1.net, aixi@deepmind.com
}

\usepackage[utf8]{inputenc}

\usepackage[mathletters]{ucs} 
\usepackage{helvet}
\usepackage{courier}
\usepackage{algorithm}
\usepackage{algorithmic}
\usepackage{amssymb}
\usepackage{amsmath}
\usepackage{wasysym}
\usepackage{amsthm}
\usepackage{stmaryrd}
\usepackage{natbib}
\usepackage{bm}
\usepackage{todonotes}

\usepackage{nomencl}
\usepackage{tikz}
\usepackage{tikz-cd}
\usetikzlibrary{shapes.geometric,arrows,fit,positioning} 
\usetikzlibrary{shapes.multipart}

\newcommand{\cX}{\mathcal{X}}
\newcommand{\cdbar}{\,|\,}
\newcommand{\cbar}{\,|\,}

\newcommand{\cA}{{\cal A}}

\newcommand{\cR}{{\cal R}}

\newcommand{\cP}{\mu}
\newcommand{\E}{\mathbb{E}}
\newcommand{\cSP}{{\cal P}}

\newcommand{\cS}{\mathcal{S}}
\newcommand{\cI}{\mathcal{I}}

\newcommand{\cZ}{\mathcal{Z}}
\newcommand{\SetN}{\mathbb{N}}
\newcommand{\SetA}{\mathbb{A}}
\newcommand{\SetB}{\mathbb{B}}
\newcommand{\SetR}{\mathbb{R}}

\newcommand{\SetD}{\mathbb{D}}
\newcommand{\strin}{\in}
\newcommand{\strnotin}{\notin}

\def\cN{{\cal N}}               
\newenvironment{keywords}{\centerline{\bf\small Keywords}\vspace{-1.5ex}\begin{quote}\small}{\par\end{quote}\vskip 1ex}

\newtheorem{thm}{Theorem}

\newtheorem{propapdx}[thm]{Proposition}

\newtheorem{defn}[thm]{Definition}

\makenomenclature

\usepackage{etoolbox}
\renewcommand\nomgroup[1]{%
  \item[\bfseries
  \ifstrequal{#1}{S}{Sets}{%
  \ifstrequal{#1}{F}{Functions}{%
  \ifstrequal{#1}{E}{Elements}{}}}%
]}

\tikzstyle{arrow} = [thick,->,>=stealth]

\tikzstyle{pol} = [rectangle, rounded corners, minimum width=3cm, minimum height=1cm,text centered, draw=black, fill=red!30]

\tikzstyle{coder} = [rectangle, rounded corners, minimum width=3cm, minimum height=1cm,text centered, draw=black, fill=blue!30]

\tikzstyle{models} = [rectangle, rounded corners, minimum width=3cm, minimum height=1cm,text centered, draw=black, fill=orange!30]

\tikzstyle{envi} = [rectangle, rounded corners, minimum width=3cm, minimum height=1cm,text centered, draw=black, fill=green!30]

\tikzstyle{sol} = [rectangle, rounded corners, minimum width=3cm, minimum height=1cm,text centered, draw=black]

\pgfdeclarelayer{background}
\pgfsetlayers{background,main}

\begin{document}

\maketitle

\begin{abstract}
Reinforcement Learning formalises an embodied agent's interaction with the environment through observations, rewards and actions. But where do the actions come from? Actions are often considered to represent something external, such as the movement of a limb, a chess piece, or more generally, the output of an actuator. In this work we explore and formalize a contrasting view, namely that actions are best thought of as the output of a sequence of internal choices with respect to an action model. This view is particularly well-suited for leveraging the recent advances in large sequence models as prior knowledge for multi-task reinforcement learning problems. Our main contribution in this work is to show how to augment the standard MDP formalism with a sequential notion of internal action using information-theoretic techniques, and that this leads to self-consistent definitions of both internal and external action value functions.
\end{abstract}

\begin{keywords}
\vspace{0.5em}
Reinforcement Learning, large action spaces, compression, coding, internal actions, sampling.
\end{keywords}

\section{Introduction}

It is hard to speak of embodied agents these days without mentioning or appealing to some notion of Reinforcement Learning.
This particular mathematical formalism has been so successful of late that the validity of its various modelling assumptions rarely gets called into question.
Yet recently we have seen a step-change in the capabilities of generative modelling, with the most striking example being in multi-modal language applications; the acquisition of gigantic multi-task datasets via internet scraping and scalable approaches to training has led to a renewed excitement for building next generation question-answering systems, chat bots, productivity tools, sentiment analysis, and in some circles, has even produced a newfound sense of optimism that the original goals of Artificial Intelligence may well be obtainable within our lifetimes.

Yet what does this mean for Reinforcement Learning? 
While its success in restricted domains is no longer in doubt, questions remain about its long-term viability as a foundational paradigm for Artificial Intelligence.
For example, effective exploration, even in restricted settings such as finite MDPs, is problematic in large unstructured state spaces, with various lower bounds demonstrating polynomial dependence on the size of the state space, e.g. \cite{strehl2009}.
While there are some noteworthy recent examples of hard exploration problems being overcome by clever heuristics \cite{ecoffet2021goexplore}, the situation in general looks challenging, if not dire.
On the other hand, recent advances in sequence modelling combined with the acquisition of gigantic datasets via internet scraping has led to a seeming step-change \cite{brown2020language} in the ability of  various types of probabilistic models to generate plausible continuations.
Is there a way to leverage this, while keeping the basic reinforcement learning formalism and derived notions such as value functions, policies, return, etc intact?

\begin{figure}[t!]
\resizebox{.99\linewidth}{!}{
\centering
\begin{tikzpicture}[node distance=2cm,style={align=center}]
\node (intpol) [pol] {Internal Policy $\pi$};
\node (ad) [coder, below of=intpol] {Arithmetic Decoder $D_\rho$};
\node (lang) [models, below of=ad] {Action Model $\rho$};
\node (extenv) [envi, below of=lang] {Environment $\mu$};
\node (adplang) [draw=black!50, fit={(ad) (lang)}] {};
\node [below=0.3cm of ad] {$+$};
\draw [arrow,dashed] (intpol.east) to [out=10, in=30, looseness=1] node[anchor=west] {Internal action\\$b\in \mathbb{B}$} (adplang.east);
\draw [arrow,dashed] (adplang) to [out=170,in=170, looseness = 1] node[anchor=east] {Internal reward / state\\$(s_t,q)\in \cS \times\mathbb{B}^{*} $} (intpol.west);
\draw [arrow] (adplang.west) to [out=190,in=190] node[anchor=east] {External action\\$a_t\in\cA$} (extenv.west);
\draw [arrow] (extenv.east) to [out=10,in=-10, looseness = 1] node[anchor=west] {External\\state / reward\\$s_t,r_t\in \cS \times\cR$} (adplang);
\end{tikzpicture}
}
\caption{Agent-environment loop with internal actions.}
\label{fig:internal_agent_env_loop}
\vspace{-1em}
\end{figure}
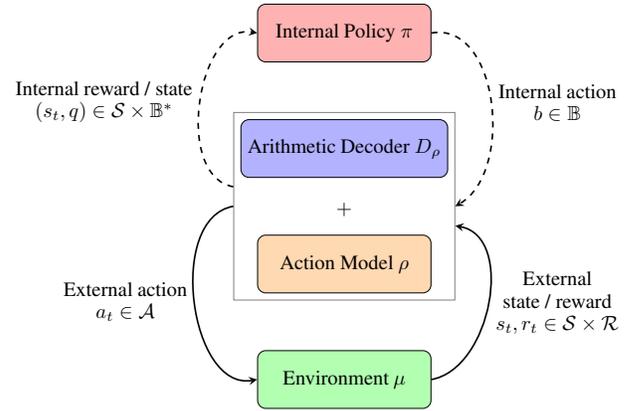

Our proposal argues for rethinking the fundamental notion of action in reinforcement learning. 
Actions are often considered to represent something external, such as the movement of a limb, a chess piece, or more generally, the output of an actuator.
In this work however, we develop a generic notion of \emph{internal} action, which is implied by a choice of action model $\rho$.
The key technical insight we leverage is the well-known duality between optimal lossless coding strategies and probabilities from information theory.
At a high level, instead of an agent directly picking an action from the action space $\cA$, instead it will pick a sequence of internal actions from an internal action set $\mathbb{B}$ which will decode to an external action from $\cA$.
Figure \ref{fig:internal_agent_env_loop} depicts this interaction graphically. 

So what do we gain by introducing this particular layer of indirection in the agent's choice of action?
Breaking up an action into a series of internal actions seems like a reasonable approach to dealing with large action spaces, and indeed has been used in other planning settings, but it immediately throws up a number of questions.
How do we decompose an arbitrary action space? 
Is there a universal, or in some sense optimal decomposition? 
When should the agent stop generating internal actions and communicate an external action to the environment? 
Does this even make sense in a reinforcement learning setting? 
How do we leverage prior knowledge in the form of a default policy?
Are there ramifications for multi-task RL?
Can we efficiently compute or sample good actions?
This paper will argue that our particular information-theoretic decomposition using an arithmetic decoder coupled with a coding distribution implied by a choice of action model naturally addresses all these questions, and opens up the possibly of leveraging recent advances in meta-learning and large-scale language/sequence models to deal with large problems using existing RL techniques.

\paragraph{Content.}
The paper is structured as follows:
Section \ref{sec:preliminaries} reviews some background material and establishes some notation; 
Section \ref{sec:internal_actions} introduces internal actions, with Section \ref{sec:internal_external_rl} formally establishing the connection between internal/external agents and environments;
Section \ref{sec:universal_action_interface} shows how the internal action framework naturally accommodates multi-task reinforcement learning settings with different action spaces.
We conclude with an extended discussion in Section \ref{sec:discussion} and cover related and future work in Sections \ref{sec:related_work} and \ref{sec:future_work}.

\section{Preliminaries}\label{sec:preliminaries}

We now briefly review the necessary background material required to describe our internal action agent-environment interaction loop.

\paragraph{Sequential Prediction.}\label{sec:compression_prediction}

A finite alphabet $\cX$ is a set of symbols.
A string of symbols $x_1x_2 \ldots x_n \in \cX^n$ of length $n$ is denoted by $x_{1:n}$.
The prefix $x_{1:j}$ of $x_{1:n}$, $j\leq n$, is denoted by $x_{\leq j}$ or $x_{< j+1}$.
The empty string is denoted by $\epsilon$.
The set of strings whose symbols come from the alphabet $\cX$ with length at most $n$ is defined by $\cX^{\leq n} := \{ \epsilon \} \cup \bigcup_{i=1}^n \cX^i$.
The set of strings of symbols from alphabet $\cX$ with finite length is denoted by $\cX^*:=\{\epsilon \} \cup \bigcup_{i=1}^{\infty}\cX^i$.
The concatenation of two strings $x$ and $y$ is denoted by $xy$.
The length of a string $x$ will be denoted by $|x|$.
We will use $y\strin x$ to denote that the symbol $y$ is in the string $x$.

A (coding) distribution $\rho$ is a sequence of probability mass functions $\rho_n : \cX^n \to [0,1]$, which for all $n\in\mathbb{N}$ satisfy the constraint that 
$\rho_n(x_{1:n}) = \sum_{y\in\cX} \rho_{n+1}(x_{1:n}y)$
for all $x_{1:n} \in \cX^n$, with the base case $\rho_0(\epsilon) := 1$.
From here onwards, whenever the meaning is clear from the argument to $\rho$, the subscript on $\rho$ will be dropped.
Under this definition, the conditional probability of a symbol $x_n$ given previous data $x_{<n}$ is defined as $\rho(x_n | x_{<n}) := \rho(x_{1:n}) / \rho(x_{<n})$ provided $\rho(x_{<n}) > 0$, with the familiar chain rules $\rho(x_{1:n}) = \prod_{i=1}^n \rho(x_i | x_{<i})$ and $\rho(x_{j:k} \cdbar x_{<j}) = \prod_{i=j}^k \rho(x_i | x_{<i})$ now following.
We will use $\Delta(\cX)$ to denote the space of probability distributions over $\cX$.

\paragraph{Arithmetic Encoding / Decoding.}
A fundamental technique known as \emph{arithmetic encoding} \citep{rissanen1979arithmetic,witten1987arithmetic} makes explicit the connection between coding distributions and source codes.
Binary arithmetic encoding is a general purpose parameterized technique that takes in a distribution $\rho$ (known as a coding distribution) and some data $x_{1:n}\in\cX^n$, and produces a uniquely decodable binary codeword $C_\rho(x_{1:n}) \in \{0,1\}^*$, whose length is essentially $\lceil -\log_2 \rho(x_{1:n}) \rceil$, which is optimal in terms of expected length if the data is sampled from $\rho$.
In essence, shorter binary codewords are assigned to data which has a higher chance of occurring under $\rho$, and longer binary codewords are assigned to the less probable data items.
Arithmetic decoding is the reverse of this procedure; it takes a coding distribution $\rho$, a binary code word $y_{1:k} = C_\rho(x_{1:n})$, and returns the original data $D_\rho(y_{1:k})=x_{1:n}$.
We will also use the shorthand notation $D_\rho(y_{1:k} \,|\, s) := D_{\rho(\cdot | s)}(y_{1:k})$ to denote decoding with respect to a coding distribution conditioned on the string $s$.
We refer the reader to the standard text of \cite{cover1999elements} for further information.

\paragraph{Markov Decision Processes.}\label{sec:background_mdps}
A Markov Decision Process (MDP) is a type of probabilistic model widely used within reinforcement learning \citep{sutton2018reinforcement,szepesvari2010algorithms} and control \citep{bertsekas1996neuro}.
In this work, we limit our attention to finite horizon, time-homogeneous MDPs whose action and state spaces are finite.
Formally, an MDP is a quadruplet $(\cS, \cA, \cR, \cP)$, where $\cS$ is a finite, non-empty set of states, $\cA$ is a finite, non-empty set of actions, $\cR \subset \SetR$ is the reward space, and $\cP$ is the transition probability kernel that assigns to each state-action pair $(s,a) \in \cS \times \cA$ a probability measure $\cP(\cdot \cbar s,a)$ over $\cS \times \cR$.
$\cS$ and $\cA$ are known as the \emph{state space} and \emph{action space} respectively.
The transition probability kernel gives rise to the \emph{state transition kernel} $\cSP(s' | s,a) := \cP(\{ s' \} \times \cR \cdbar s,a)$, which gives the probability of transitioning from state $s$ to state $s'$ if action $a$ is taken in $s$.

An agent's behavior is determined by a \emph{policy} that defines, for each state $s\in\cS$ and time $t \in \mathbb{N}$, a probability measure over $\cA$ denoted by $\pi_t(\cdot \cdbar s)$.
A \emph{stationary policy} is a policy which is independent of time, which we will denote by $\pi(\cdot \cdbar s)$ where appropriate. 
At each time $t$, the agent communicates an action $A_{t} \sim \pi_t(\cdot \cdbar S_{t-1})$ to the system in state $S_{t-1} \in \cS$.
The system then responds with a state-reward pair $(S_{t},R_{t}) \sim \cP(\cdot \cdbar S_{t-1}, A_{t})$.
Here we will assume that each reward is bounded between $[r_{\min},r_{\max}] \subset \mathbb{R}$ and that the system starts in a state $s_0$ and executes for an infinite number of steps.
Thus the execution of the system can be described by a sequence of random variables $S_0, A_1, S_{1}, R_{1}, A_2, S_2, R_2, ...$.

The finite $m$-horizon \emph{return} from time $t$ is defined as $Z_{t} := \sum_{i=t}^{t+m-1} R_i$.
The expected $m$-horizon return from time $t$, also known as the \emph{value function}, is denoted by $V^{\pi}_{\mu}(s_t) := \E [Z_{t+1} \cdbar S_t =s_t]$.
The return space $\cZ$ is the set of all possible returns.
The \emph{action-value function} is defined by $Q^{\pi}_{\mu}(s_t, a_{t+1}) := \E [Z_{t+1} \cdbar S_t =s_t, A_{t+1}=a_{t+1}]$.
An \emph{optimal policy}, denoted by $\pi^*_{\mu}$, is a policy that maximizes the expected return $\E \left[ Z_{t+1} \cdbar S_t \right]$ for all $t$.

\section{Information-Theoretic Actuation -- Internal Actions}
\label{sec:internal_actions}

We now describe in detail how to combine the aforementioned building blocks into the internal reinforcement learning framework described in Figure \ref{fig:internal_agent_env_loop}, and discuss its ramifications.
Compared with the standard agent-environment loop, there are two additional components with this setup:
a choice of action model $\rho$, and an associated arithmetic decoder $D_\rho$ that uses $\rho$ as a coding distribution.
The internal action space $\mathbb{B}$ is defined by the associated decoding alphabet used by $D_\rho$; for example, using a binary arithmetic decoder would lead to an internal action space of $\mathbb{B} = \{ 0, 1 \}$.
For pedagogical purposes, we will restrict our attention to this case in the rest of the paper, but remark that any finite decoding alphabet can in principle be used with our construction.

We first introduce our notion of internal action.
At a high-level, one should think of a single internal action as a bit-commitment towards a particular choice of external action, with particular sequences of these corresponding to external actions. 
In a sense, internal actions correspond to a period of private deliberation by the agent, which upon conclusion produces a string describing the desired actuation in compressed form; in essence, the arithmetic decoder functions as a universal actuator, whose behavior can be completely configured by a choice of action model.

\paragraph{Reshaping of the action space.}
As alluded to before, the effect of the action model is to reshape the action space, which the following example will make clear.
Figure \ref{fig:chess1} shows an illustrative example of the behavior of a binary arithmetic decoder equipped with an action model based on a GLN-based context mixing language model \cite{veness2019gated} that has been pre-trained on 9MB of grandmaster games in PGN (Portable Game Notation) format.
On the left hand side of the table, we have the input to the decoder, and on the righthand side we have the decoded output; if we consider the first row, the LHS corresponds to the bitstring $10 = C_\rho(\text{a6})$ and the RHS corresponds to $D_{\rho}(10)$, with $\rho$ here denoting our pre-trained language model.
The LHS of the first 4 rows shows the encoding of a natural sequence of continuing moves (known as the Morphy Defense), while the last four rows show an illogical continuation of moves which ignore development, lose castling rights, and even hang the queen.
One can see that much shorter codes are assigned to the more logical sequence of moves.
This shows the effect of the action model as providing a type of inductive bias, which we will discuss in greater depth later.

In contrast, one could also consider the effect of a completely uninformative action model, $\rho_{\text{\sc uniform}}(a | s) := 1 \, / \, |\cA|$, which assigns uniform probability mass to each possible external action in every state.
Here every single action would have the same codelength of $\lceil \log |\cA| \rceil$, which would correspond to a naive binarization of the external action space.

\paragraph{When to stop decoding.}

\begin{figure}[h!]
\resizebox{.99\linewidth}{!}{
\begin{centering}
\begin{tabular}{ |l|l| } 
\hline
\bf Input bits & \bf Decoded Output \\
\hline
\small{10} & \small{a6}                                                                   \\
\small{10010} & \small{a6 Ba4}  \\
\small{100100} & \small{a6 Ba4 Nf6}   \\
\small{1001010111} & \small{a6 Ba4 Nf6 O-O}   \\
\hline
\small{010010101010011} & \small{Nh6}  \\
\small{0100101010101000110000010} & \small{Nh6 Kf1}  \\
\small{01001010101010001100000110010010101} & \small{Nh6 Kf1 Qg5} \\
\small{01001010101010001100000110010010101010010010001} & \small{Nh6 Kf1 Qg5 Na3} \\
\hline
\end{tabular}
\end{centering}
}
\caption{Arithmetic decoding example. Some example decoded outputs from a pre-trained model on chess, with the model's context set to the Ruy Lopez opening, namely: \emph{e4 e5 Nf3 Nc6 Bb5}. }\label{fig:chess1}
\end{figure}

Figure \ref{fig:chess1} also highlights a technical issue which we need to resolve, namely, how and when is a decoded action to be transmitted to the external environment?
For example, if we wanted a chess-playing agent whose action space was the space of single moves, we need some way to know when our decoded output should be communicated to the environment as an external action.
Although other solutions are possible, in this work we adopt the convention that every external action can be described as a string formed by the concatenation of atomic symbols from a common alphabet.
More formally, we assume that the action space $\cA \subseteq \SetA^{\le k}$, where $\SetA$ denotes the sub-action alphabet, and $k$ is a positive constant.
We assume that the sub-action alphabet always contains a privileged termination symbol $\top \in \SetA$, which has the semantics that when it is decoded it causes an external action to be communicated to the environment.
Note that in finite action/state MDPs, this modification does not impose any restrictions nor add further expressive power.
Returning to the example shown in Figure \ref{fig:chess1}, by identifying the space character with $\top$, we would know when to transmit an external action. 
This is implemented formally via a function $\tau:\SetA^{\le k}\to \cA$ which takes actions and returns the action component up to but not including the first $\top$, for example $\tau(\text{a6}\,\top)=\text{a6}$. 
This importantly handles the case of multiple $\top$ symbols, for example $\tau(\text{a6}\,\top\, \text{Ba4}\,\top)=\text{a6}$.

A terminal symbol is not the only way to know when to stop decoding. Another approach could be to only allow prefix-free codes. This will however run into it's own problems, such as what prefix-free encoding to use, how to enumerate the elements of $\cA$ so that the corresponding prefix code can be found easily (and vice versa). Using an ``optimal'' prefix code would require the use of universal Turing machines and is beyond the scope of this paper.
Another choice to stop decoding is to consider the action before the last $\top$ symbol, instead of before the first. In this case the agent may take multiple actions without knowing the state in between them.

\paragraph{Internal action loop.}
External action selection is determined by executing our internal policy $\pi$ until the concatenation of these binary actions uniquely decodes into an external action. 
Once the action model and arithmetic decoder have generated an external action, this external action will be sent to the external environment. 
The external environment will then return an external observation/reward to the action model and arithmetic decoder combination, and the internal policy receives a reward $r_t$ from the external environment. 
This interaction is displayed graphically in Figure \ref{fig:internal_agent_env_loop} and described procedurally by Algorithm \ref{fig:ag_env_int}.
\begin{algorithm}
\begin{algorithmic}
\caption{Internal Agent-Environment loop} \label{fig:ag_env_int}
\medskip
\REQUIRE Internal policy $\pi  : \cS \times \SetB^{*} \to \Delta \SetB$ 
\REQUIRE External environment $\mu: \cS \times \cA \to \Delta(\cS \times \cR)$
\REQUIRE Action model $\rho: \cS \to \Delta \SetA^{\le k}$
\medskip
\FOR {$t=1,2,3,\ldots$} 
\STATE Observe $s_t,r_t\sim\mu(\cdot,\cdot|s_{t-1},a_{t-1})$ 
\STATE $a_t,q= \epsilon$ 
\WHILE {$\top \strnotin a_t$}
\STATE $b \sim \pi(\cdot |s_t,q)$ 
\STATE $q = qb$ 
\STATE $a_t =D_\rho(q|s_t)$ 
\STATE $r_t=0$
\ENDWHILE
\STATE $a_t=\tau (a_t)$
\STATE Act $a_t$
\ENDFOR
\end{algorithmic}
\end{algorithm}

\paragraph{Example.}
We conclude this section with an example execution in the context of our previous chess example. Here the state space is the finite history of moves.
At the first time-step, the system receives an observation from the external environment. This observation becomes the current state
$ s_0 = \text{e4 e5 Nf3 Nc6 Bb5} $. 
Then given this state the internal policy takes internal binary actions one by one
$ b_0=1,b_1=0$.
The internal policy continues to take internal binary actions until the sequence of internal binary actions uniquely decodes into an external action $a_t$. 
A concatenation of characters is an external action when it has a ``terminal'' symbol $\top$, 
$ \text{a6} = D_\rho(b_0 b_1 |s_0) $.
Now that we know the external action, it is sent to the external environment $\mu$, 
$ a_1 = \text{ a6 } $.
Then we receive the next observation from $\mu$,
\[ o_1 \sim \mu (\cdot | s_0 a_1) = \mu (\cdot | \text{e4 e5 Nf3 Nc6 Bb5 a6}) \]
and now $ s_1 ~:=~ s_0 a_1 o_1$ 
and the process repeats.

\section{Connecting External to Internal Agents and Environments}
\label{sec:internal_external_rl}

In this section, we will describe formally how to augment an arbitrary external environment to an internal environment that the internal action agent is able to interact with; additionally if the external environment is Markovian then the internal environment will also be Markovian. 
Our approach will be to construct an augmented environment $\vartheta$, called the internal environment, comprised of the true environment, the action model and the arithmetic decoder.
We will also show how the internal action agent can be uplifted to an external agent, and then show that both the internal environment with an internal action agent and the external environment with the uplifted policy are equivalent in the sense that they achieve identical action-value functions. 
These results will allow for easier analysis of the internal action agent setup, as well as the ability to apply any result or algorithm specific to MDPs to the internal agent setup.

\paragraph{Internal Environment.}
The internal environment $\vartheta$ is a stochastic function over internal states and internal actions to internal states and rewards. The internal state space used here will be $\cI := \cS \times \SetB^{\le n} $, the state from the external environment and previous internal actions taken by the internal agent, until they are decoded to an external action. 
We consider the finite set $\SetB^{\le n}$ over the infinite set $\SetB^{*}$, as for any external action $a$ with $\rho(a)> 0$ there will always be a finite number of binary actions needed to decode $a$; $n$ is the maximum of those finite numbers.
We will use $\top$ to denote the ``terminal'' symbol, that is, the symbol that indicates when the concatenation of internal actions corresponds to a complete external action, and is sent to the external environment. 
We will use the symbols $s,s'$ for elements of $\cS$, the first component of the internal state. 
We will use $q,q'$ for elements of $\SetB^{\le n}$, the second component of the internal state, the internal agent's previous internal actions.
The symbol $b$ will be used for the internal agent's internal action. The symbol $a$ will be used for a decoded external action, e.g. $D_\rho(qb|s)=a$. 
The true external environment will be denoted by $\mu$, which is a stochastic function from external states and external actions to external states and rewards. The external state space is $\cS$. The external action space is $\cA\subseteq \SetA^{\le k}$.

\begin{defn}[\bf\boldmath MDP (Internal) Environment $\vartheta$]
The internal policy $\pi$ interacts with an internal environment $\vartheta : \cI  \times \SetB \to \Delta (\cI   \times \cR ) $ which is defined by the action model $\rho$ (encoder/decoder $C_\rho$/$D_\rho$ generated by $\rho$) and the true external environment $\mu$ as follows:
\begin{align*}
    & \vartheta (s'q'r|sq,b) := \\
    &\begin{cases}
     \mu (s' r|s, \tau(a) )\!\! &\text{ if }q'=\epsilon\  \land
      (\top\strin a ), \\
     1 &\text{ if } s'=s\ \land
      q'=qb\land  r=0\ \land
      (\top\strnotin a), \\
    0 &\text{ otherwise}
    \end{cases}
\end{align*}
where $a:=D_\rho(qb|s)$, $(s'q',r)\in \cI   \times  \cR  $, $sq\in \cI$ and $b\in \SetB$. \label{def:internal_env}
\end{defn}

The definition of $\vartheta$ is split up into three cases: In the first case the decoded $qb$ contains the symbol $\top$, $\top\strin a$ where $a:=D_\rho(qb|s)$, and the previous binary characters $q'$ resets to being the empty string $\epsilon$. 
In this case the $\tau$ of the decoded action $D_\rho(qb|s)$ is sent to the external environment $\mu$, and the next state $s'$, is the external state $s'$. 
The second case of $\vartheta$ is when the internal agent is still decoding, that is, $\top\strnotin a $ and the next state $s'q'=sqb$ is updated by the agent's action $b$, and the internal reward $r$ is 0. 
In the third case, where neither set of above conditions is satisfied, the probability of the state $s'q'$ and reward $r$ is 0. 
In this way the environment $\vartheta$ is deterministic during the decoding process, and only stochastic when it sends the decoded action to the external environment.

Given the internal agent's policy $\pi$ and the arithmetic decoder $D_{\rho}$, we can construct an external policy $\Pi$ which will interact with the true external environment $\mu$. The external policy $\Pi$ is a stochastic function from external states $s\in \cS$ to external actions $a\in \cA$. 
To construct $\Pi$, we consider all possible binary strings $q\in \SetB^{\le n}$ such that the arithmetic decoder will decode $q$ into $a$ given $s$.
For this we will need to define a \textit{decodable} subset of $\SetB^{\le n}$. We will use $\SetD$ to denote the set of decodable binary strings. A string $q$ is decodable if $\top$ is in the decoding of the string, and $\top$ is not in the decoding of the first $|q|-1$ elements of the string. Formally this means
\[ \SetD_s := \left\{ q\in\SetB^{\le n} :\top\strin D_{\rho}(q|s)\ \land \  \top\strnotin  D_{\rho}(q_{<|q|}|s)  \right\} . \]
We then consider the probability that $\pi$ will output the internal binary actions that eventually construct $q$, which using the chain rule we can write as the product of probabilities that $\pi$ will take the action of each element of $q$ given the previous elements of $q$. All together this is written as follows: 
\begin{equation}
     \Pi(a|s) := \sum_{\substack{q\in \SetD_s : \\  a=\tau(D_{\rho}(q|s))}} \prod_{i=1}^{|q|} \pi(q_i|sq_{<i}). \label{eq:uplifted_pi}
\end{equation}
It is important to note that there may be more than one binary string $q\in \SetD_s$ such that $a=\tau(D_{\rho}(q|s))$; this comes from how arithmetic decoders work. 
For example, consider a case where 
\begin{align*}
    &D_\rho (10|s)=e,\ D_\rho (100|s)=e4,\ D_\rho (101|s)=e4 \\
    &D_\rho (1000|s)=e4\ c5,\ D_\rho (1001|s)=e4\,\top \\
    &D_\rho (1010|s)=e4\,\top,\ D_\rho (1011|s)=e4\ e5
\end{align*}
We have that both $1001$ and $1010$ are elements of $\SetD_s$ and both $\tau(D_{\rho}(1001|s))=e4$ and $\tau(D_{\rho}(1010|s))=e4$, therefore $\Pi(e4|s) $ would be a sum over $1001$ and $1010$.

\paragraph{Self-consistency of internal and external Q-values.}

We can use the external agent $\Pi$ to interact with the external environment $\mu$, just as any regular RL agent would.

\begin{thm}[Internal/External value equivalence]
For all states $s\in\cS$, previous internal actions $q\in\SetB^{\le n}$, external actions $a\in\cA$ and internal actions $b\in\SetB$, if $\tau(D_{\rho}(qb|s))=a$ then
\begin{equation}
   Q_{\mu}^{\Pi} (s,a) = Q_{\vartheta}^{\pi}(sq,b) 
    \label{eq:q_value_equiv}.
\end{equation} 
That is, the action-value function for the external policy $\Pi$ and external environment $\mu$ is equal to the action-value function for the internal policy $\pi$ and the internal environment $\vartheta$. 
\end{thm}

\begin{proof}
This proof comes from expanding the action-value function using Equation \ref{eq:uplifted_pi} and Definition \ref{def:internal_env} to rearrange the expanded action-value function. For the full proof see the supplementary material.
\end{proof}

Because of Equation \ref{eq:q_value_equiv} we are able to say that if an internal agent $\pi$ performs well, in the sense of a high action-value, in the internal environment $\vartheta$, then the uplifted version of the agent $\Pi$, performs well in the true external environment $\mu$.

\section{A Universal Action Interface \\for Multi-task RL}
\label{sec:universal_action_interface}

A key complication and limiting factor in the design of any multi-task RL system is how to deal with the potentially radically different action spaces required for each distinct task.
While it is feasible to make a generic agent work well across multiple similar domains e.g. Atari games \cite{mnih2015humanlevel}, the situation becomes considerably more complicated when the action spaces of the different tasks vary dramatically.
The arithmetic encoding-based approach we advocate provides an elegant solution to this problem, which builds on techniques from universal source coding.

Given $K > 1$ coding distributions, it is straightforward to combine them into a universal ensemble whose compression performance will be close to that of the best coding distribution in hindsight.
If we denote the $i$th coding distribution by $\rho_i$, one can take a uniform Bayesian mixture of the $K$ coding distributions, whose marginal distribution over sequences is given by
\begin{equation}
\xi(x_{1:n}) := \sum_{i=1}^K \frac{1}{K} \; \rho_i(x_{1:n}).
\end{equation}
A standard dominance argument shows that the logarithmic loss/coding length of the mixture $\xi$ compared to any choice of action model $j$ is bounded by
\begin{align*}
-\log \xi(x_{1:n}) &\leq -\log \left( \frac{1}{K} \; \rho_j(x_{1:n}) \right) \\
&= -\log \rho_j(x_{1:n}) + \log K,
\end{align*}
or in other words, the excess log-loss is bounded by a constant, which is asymptotically negligible when one considers the time-averaged performance of the ensemble.

This has important ramifications for multi-task reinforcement learning in our internal action formulation.
Recently, various works \cite{janner2021reinforcement} have attempted to frame reinforcement learning in terms of probabilistic sequence models over interaction strings, i.e. defining a sequential probability measure $\nu$ over strings that represent state/reward/action histories in the form $s_1 r_1 a_1 \dots$.
By taking a uniform Bayesian mixture over multiple instances of these history-based measures for different tasks, just as in the coding distribution example, one also obtains a sequence model that is universal across all of these tasks.
More formally, given a history string $h$ which is an element of $(\cS \times \cR \times \cA)^* \cup \left( (\cS \times \cR \times \cA)^* \times (\cS \times \cR)\right)$, we can define the uniform Bayesian mixture
\begin{equation}\label{eq:history_mix}
\xi(h) = \sum_{i=1}^K \frac{1}{K} \, \nu_i(h)
\end{equation}
over $K$ history based measures $\nu_i$, with each $\nu_i$ corresponding to a task specific history model.
Note that this formulation in terms of measures on strings still implies the usual Bayesian learning in terms of sequential updating of the posterior, it is just hidden in this notation; see Section 2 by \citet{veness2016} for a brief overview.

An interesting effect now emerges if we use the conditional action distribution $\xi(\cdot | s_1 r_1 a_1, \dots ,s_n r_n)$ as the action model in our setup.
In particular, this action model will rapidly learn to \emph{automatically generate actions appropriate for the underlying task}, without requiring any task identity information.
How this works is subtle; Bayesian inference is used implicitly by $\xi$ to determine which task the agent is most likely in, and due to the rapid convergence of the Bayesian mixture to the best task specific model, the action model used for decoding after a small number of external environment interactions will essentially behave the same as if we knew which task specific action model to use in the first place.
In other words, what this means in practice is that one can use $\xi$ as the action model, and $C_\xi$ will produce codes which are almost as short as any task-specific action encoding $C_{\rho_j}$.
In particular, this implies that short bitstrings can decode to very different external actions which are plausible under either task-specific model.

The most interesting aspect about this construction is that the internal action formalism allows us to treat a multi-task reinforcement problem as a single reinforcement learning task with a common action space.

\section{Discussion}\label{sec:discussion}

This section discusses some interesting and potentially surprising ramifications of information-theoretic internal actions.

\paragraph{Uninformative internal policy $\pi$ and application to Large Language Models.}
An interesting corollary of the internal reinforcement learning setup is that a uniform policy over internal actions gives rise to an external agent that selects actions that are essentially distributed to the action model.
This is a by-product of the duality between optimal codes and probabilities, and can be seen with the following argument.
In the case of a binary internal action space, a uniform policy will generate a particular sequence $b_{1:n}$ with probability $2^{-n}$.
Now, notice that the probability of an action $a \in \cA$ in state $s$ where $C_{\rho}(a | s) = b_{1:n}$ is given by
\begin{equation*}
\rho(a | s) = 2^{\log_2 \rho(a | s)} \approx 2^{-|C_{\rho}(a | s)|} = 2^{-n}.
\end{equation*}
The approximate equality step is due to the small gap between the optimal code length $-\log_2 \rho(a | s)$ and the realised code length $C_{\rho}(a | s)$ produced by an arithmetic encoder coupled to $\rho$.
The size of this gap is bounded by 1, but is essentially negligible for the purposes of this argument. 

This has interesting ramifications for constructing agents when the action model is already useful, such as in the case of Large Language Models (LLMs) obtained by supervised learning on massive amounts of internet data.
Random behaviour by the internal policy in this case will still produce useful behavior, which provides a natural starting point for any internal policy learning technique. 
In particular random internal-action exploration becomes targeted, possibly leading to optimal external-action exploration, in a similar fashion to Thompson Sampling.
It is in this way that the action model provides a powerful mechanism for specifying data-dependent prior knowledge to existing reinforcement learning algorithms.

\paragraph{Specifying the action space from data.}
In complicated environments, it may be difficult or complicated to precisely specify the action space explicitly.
This situation readily arises in natural language domains for example.
In these cases it is more natural to simply learn a probabilistic model of the domain.
Our internal agent formalism directly allows for this possibility via the action model.
The action model allows for a strict separation between pre-training on data, for example pre-training an action model using a collection of grandmaster games in chess, and the resultant learning behavior of the internal agent. 

It is also worth pointing out an interesting connection to meta-learning with sequence models across many tasks.
Perhaps surprisingly, perplexity-based meta-learning of history-dependent LLMs is closely related to the explicit Bayesian mixture solution described in Equation \ref{eq:history_mix}.
In particular, one can show that in many standard meta-learning setups, the optimal perplexity-minimizing solution is \emph{exactly} a Bayesian mixture distribution \cite{ortega2019metalearning}. 
Provided that a sufficiently powerful history-dependent model is used (such as the case with LLMs based on Transformers) to model the interaction histories, a low-perplexity solution can be seen as a learnt approximation to the explicit Bayesian construction we provided in in Equation \ref{eq:history_mix}. 
In this way the action space for a multi-task agent can be learnt directly from data alone, which goes some way to explaining the recent empirical success of approaches such as \citet{janner2021reinforcement}.
In other words, if one wanted an agent that could play both chess and something with a radically different action space such as the text-based NetHack \cite{kuttler2020nethack}, a natural strategy would be to pre-train using meta-learning a sequence model from example trajectories in both games, and use this to define the action model; the internal action framework will then automatically deal with the different underlying action spaces.

\paragraph{Comparison to binarization.}
It is instructive to consider the differences between a direct binarization of the action space compared with our approach.
One can interpret the combination of an action model and a binary arithmetic decoder as a generalized form of binarization.
As discussed earlier, an action model which assigned a uniform distribution over the action space in every state is equivalent to a naive binarization, with every action being assigned a code-length of $\lceil \log_2 |\cA| \rceil$.
Binarization of actions in reinforcement learning has the obvious benefit of reducing the size of the action space, which can lead to some benefits \cite{majeed2020exact}. However, often this binarization comes with a corresponding increase to the planning horizon, and in many circumstances provides no benefits.

Our non-uniform binarization essentially reshapes the action space according to the knowledge contained within the action model.
Thus planning using various types of depth limited search takes on a different meaning in our internal reinforcement learning setting.
Although the planning algorithm may only be searching $d$ steps ahead, the implied information-theoretic planning horizon might be much greater than $d$.

\paragraph{Computational advantages.}
Many reinforcement learning techniques require an ability to efficiently generate a random sample from the action space.
A convenient property of our formalism is that it provides a generic technique to generate samples from arbitrary action models/action spaces.
This is a byproduct of having an arithmetic decoder coupled to an action model.
By generating a sequence of bits $y_{1:m}$ with each bit sampled from a $\text{Bernoulli}(1/2)$ distribution, 
and feeding them to a binary arithmetic decoder $D_\rho$ coupled to the action model $\rho$, 
one can show that external action $a:=D_{\rho}(y_{1:m})$ is distributed according to $\rho$ \cite{MacKay2003},
which resembles Thompson sampling.
We can also efficiently compute the probability of $a$ as a product of conditional probabilities ($P[a]=\prod_{t=1}^k P[b_t|b_{<t}]$) required for some learning algorthms.
We can even efficiently compute the cumulative probability based on the recursion $P(X_{1:k}\leq b_{1:k})=[\![b_1=1]\!]P(X_1=0) + P(X_{2:k}\leq b_{2:k}|X_1=b_1)P(X_1=b_1)$.
Unfortunately binarization does \emph{not} lend itself to an efficient way of computing the most probable (MAP) action $\arg\max_a P(a)$.
But Thompson sampling for large spaces often performs better than MAP anyway, since the latter is not representation invariant and favors brittle solutions.
Binarization decreases the branching factor in planning algorithms but increases the planning horizon. Since binarized actions are length-optimized this \emph{may} still lead to a net win.
For example, depth-limited planning techniques, which typically have an exponential dependence on the length of the horizon, now have an exponential dependence on the combined code-length under $\rho$, which drastically alters their semantics and is closely connected to using a prior policy to guide search such as in successful approaches for Computer Go \cite{silver2017mastering,orseau2018single,orseau2021policy}.

\paragraph{Pre-training and universality.}
A common use case in machine learning is to consider fine tuning an existing pre-trained model to save on compute.
The next result shows that pre-training on any data will not affect the asymptotic performance of any consistent density estimator.
In our context, it suggests that a good general approach to constructing an action model for a new domain might be to first pre-train on large, task-agnostic data and then to use fine tuning to incorporate task-specific knowledge if this data is available.

More formally, consider sequences $X_{1:\infty}$ over a finite alphabet $\cX$ sampled from $P_{\theta_0}$.
Assume $\theta(X_{1:n})$ is a consistent estimator of $\theta_{0}$.
Then whatever the first $k$ samples $x_{1:k}$ are,
$\theta(x_{1:k}X_{k+1:n})$ is still a consistent estimator of $\theta_0$. Additionally, the reverse is also true. 
Most importantly, this holds without \emph{any} assumptions on the stochastic process $(X_t)\sim P_{\theta_0}$.
\begin{propapdx}[\bf consistency is immortal]
\label{thm:consis_imm_iff}
For any fixed $k\in\mathbb{N}$, $\theta(X_{1:n})$ is a consistent estimator of $\theta_0$ if and only if for all $x_{1:k}$ such that $P_{\theta_0}[x_{1:k}]>0$, $\theta(x_{1:k}X_{k+1:n})$ is a consistent estimator of $\theta_0$. 
\end{propapdx}

One consequence of this proposition is that for a given $\rho$ and $\pi$, if $\Pi$ defined in Equation \ref{eq:uplifted_pi} is a consistent estimator of $\pi_{\mu}^*$, the optimal policy in $\mu$, then if the action model $\rho$ is pre-trained on additional data then $\Pi$ is still a consistent estimator.

\paragraph{Discounting.}
One subtlety that arises is how to best communicate a reward from the environment to the internal policy.
Notice that in Algorithm \ref{fig:ag_env_int}, after the first internal action, $r_t$ is set to 0.
This has the effect of preserving the return and associated discounting schedule in the external environment.
Although in the case of uninformative action binarization one can map a discounted external setup to an equivalent discounted internal setup \cite{majeed2020exact}, attempting a more general construction along those lines in our case requires time-dependent, and worse, history-dependent discounting which runs into both technical and computational challenges.

\section{Related Work}
\label{sec:related_work}

\vspace{0.5em}
We now discuss some related work.

\paragraph{Compression-based RL.}
Using compression to aid with machine learning goes back to at least the work by \citet{frank2000text}, where compression-based models were compared to classical machine learning methods on a number of natural language problems. More recently, \citet{hamilton2013modelling} used compression with Predictive State Representation on domains with large observation spaces to aid with the intractability. \citet{botvinick2015reinforcement} discussed the internal representation of reinforcement learning, specifically a natural or efficient coding of the internal representation. \citet{veness2015compress} used compression-based techniques for policy evaluation via action-value estimation.

\paragraph{Large transformer/language models used to aid RL.}
Language models have had a recent resurgence, starting with \textit{Attention is all you need} \cite{vaswani2017attention} and being followed by the success of GPT-2 \cite{radford2019language} and GPT-3 \cite{brown2020language}. Despite the accomplishment of language models on NLP tasks, there have only been a few circumstances where these techniques have translated to the field of reinforcement learning. \citet{luketina2019survey} provides a survey of reinforcement learning methods which have been improved with the addition of natural language approaches. One such example is \citet{kaplan2017beating}, where natural language methods were used with deep reinforcement learning to play Atari games. In addition to this, recent work on using transformer models with reinforcement learning includes: \citet{parisotto2020stabilizing}, \citet{noever2020chess} train GPT-2 \cite{radford2019language} on the PGN format to learn chess, 
\citet{ciolino2020go} trained GPT-2 in a similar way to learn Go, and \citet{stein2020stabilizing} used Transformers for Deep Q-learning to play Atari games. \citet{krause2020gedi} introduced a coding scheme to improve small language models.

\section{Future Work}
\label{sec:future_work}

Arithmetic encoding has a number of extensions which deserve further investigation in the context of reinforcement learning. 
In particular, one can generalise arithmetic encoding to time-adaptive coding distributions, which is known as adaptive arithmetic encoding. 
While one could crudely incorporate this notion into our existing work, a more complete treatment would require going outside the MDP formalism.

Finally from a theory perspective, there are some additional generalizations that can be made, though they are beyond the scope of this paper. 
These include a more thorough treatment of the infinite horizon case by using discounting, investigating setups which do not require an end of action symbol $\top$ (as discussed earlier) and allowing the external action space to be countably infinite.

\section{Conclusion}

In this work we have laid the conceptual foundations for information-theoretic actuation.
We revisited the meaning of action in reinforcement learning, and explored a particular type of internal viewpoint.
We have demonstrated how our method is theoretically well justified, by formally connecting it to a traditional reinforcement learning MDP setup.
We argued that such a framework is well positioned to take advantage of the recent progress in large sequence/language models for multitask reinforcement learning problems over large action spaces.
The next step is to explore application of this formalism in conjunction with modern sequence modelling techniques on some benchmark problems to better understand the potentials and limitations of this approach.
Multi-task RL problems with vastly different action spaces seem the most natural setting where our approach could have immediate impact.

\begin{small}
\appendix
\end{small}

\section{Acknowledgements}
We thank Csaba Szepesvari, Nando de Freitas, Oriol Vinyals, and Pedro Ortega for some helpful discussions.
This work has been supported in parts by the Australian Research Council under grant DP150104590.

\bibliography{cac.bib}

\nomenclature[F]{$\rho:\cS\to \Delta \cA$}{Action Model}
\nomenclature[F]{$D_{\rho}(\cdot\mid s):\SetB^{\le n}\to \mathbb{A}^{\le k}$}{Arithmetic decoder with action model $\rho$}
\nomenclature[F]{$C_{\rho}(\cdot \mid s):\cA\to \SetB^{\le n}$}{Arithmetic coder with action mode $\rho$}
\nomenclature[S]{$\cA\subseteq \SetA^{\le k}$}{Action Space}
\nomenclature[S]{$\cS$}{State space}
\nomenclature[S]{$\SetN$}{Set of Natural numbers}
\nomenclature[S]{$\SetD$}{Set of decodable binary strings}
\nomenclature[E]{$t\in\SetN$}{(External) Time}
\nomenclature[E]{$i\in\SetN$}{(Internal) Time}
\nomenclature[E]{$m\in\SetN$}{Finite horizon}
\nomenclature[F]{$V:\cS\to \mathbb{R}$}{Value function}
\nomenclature[F]{$Q:\cS\times \cA \to \mathbb{R}$}{Action-value function}
\nomenclature[F]{$\mathbb{E}$}{Expectation}
\nomenclature[S]{$\cZ$}{Return space}
\nomenclature[E]{$a,a',a_t\in\cA$}{External action}
\nomenclature[F]{$|\cdot |:\cX^* \to \mathbb{N}$}{Length of a string}
\nomenclature[S]{$\SetA$}{Finite alphabet}
\nomenclature[S]{$\SetB = \{0,1\}$}{Binary Set}
\nomenclature[E]{$b\in\SetB$}{Internal action}
\nomenclature[E]{$s,s'\in\cS$}{(External) states}
\nomenclature[E]{$n\in\SetN$}{Maximum length of binary action component of internal state}
\nomenclature[E]{$k\in\SetN$}{Maximum length of external actions}
\nomenclature[E]{$q,q'\in \SetB^*,\SetB^{\le n}$}{Finite string of binary characters}
\nomenclature[F]{$\pi:\cI\to\Delta\SetB$}{Internal policy}
\nomenclature[S]{$\cR$}{Reward set}
\nomenclature[E]{$r,r_t\in\cR$}{Rewards}
\nomenclature[F]{$\mu:\cS\times\cA\to\Delta(\cS\times \cR)$}{(True) External environment}
\nomenclature[E]{$\top\in\SetA$}{Stop/end symbol}
\nomenclature[F]{$\Pi:\cS\to\Delta \cA$}{Uplifted internal policy $\pi$}
\nomenclature[F]{$\vartheta:\cI \times \SetB \to \Delta (\cI \times \cR)$}{Internal environment}
\nomenclature[E]{$\epsilon$}{Empty String}
\nomenclature[S]{$\cX$}{Finite alphabet}
\nomenclature[E]{$x,y\in\cX^n$}{Element of finite alphabet}
\nomenclature[S]{$\SetR$}{Set of real numbers}
\nomenclature[F]{$\xi:\cS\to \Delta \cA$}{Bayesian mixture of action models}
\nomenclature[S]{$\SetD_s$}{Set of decodable binary strings}
\nomenclature[S]{$\cI = \cS\times \SetB^{\le n}$}{Internal state space}
\nomenclature[F]{$\tau:\SetA^{\le k}\to \cA$}{Return everything before first $\top$ function}

\clearpage
\onecolumn
\printnomenclature

\clearpage

\appendix
\section*{Supplementary Material}

\paragraph{Proof of Theorem 2}

\begin{proof}
Instead of expanding the entire action-value function we will show that a single interaction between $\mu$ and $\Pi$ are equal to an interaction between $\vartheta$ and $\pi$. Let $s',r,a'$ be arbitrary states, rewards and actions, then
\begin{align*}
&\mu(s'r|s,a) \Pi(a'|s') \\
&\stackrel{(a)}{=} \mu(s'r|s,a)  \sum_{\substack{q'\in \SetD_{s'} : \\  a'=\tau(D_{\rho}(q'|s'))}} \prod_{i=1}^{|q'|} \pi(q_i'|s'q_{<i}') \\
&\stackrel{(b)}{=} \vartheta(s'\epsilon r|sq,b) \sum_{\substack{q'\in \SetD_{s'} : \\  a'=\tau(D_{\rho}(q'|s'))}} \prod_{i=1}^{|q'|} \pi(q_i'|s'q_{<i}') \\
&\stackrel{(c)}{=} \sum_{\substack{q'\in \SetD_{s'} : \\  a'=\tau(D_{\rho}(q'|s'))} } \vartheta(s'\epsilon r|sq,b)\pi(q_1'|s'\epsilon) \times  \\
&~~~~~~~~~~~~~~~~~~~~\prod_{i=2}^{|q'|}\vartheta(s'q_{<i}' 0|s'q_{<i-1}',q_{i-1}')  \pi(q_i'|s'q_{<i}').
\end{align*}

Step (a) comes from Equation \ref{eq:uplifted_pi}; step (b) comes from the first case of Definition \ref{def:internal_env}; step (c) comes from the second case of Definition \ref{def:internal_env}, where $ \vartheta(s'q_{<i}' 0|s'q_{<i-1}',q_{i-1}') =1$.

Therefore in the expansion of the action-value function we can replace one with the other,

\begin{align*}
    Q_{\mu}^{\Pi} (s,a) &= \E_{\mu}^{\Pi} \!\left[\sum_{t=1}^m r_t|s,a\right]\\
    &= \sum_{s_1,r_1,a_1} \mu(s_1r_1|s,a) \Pi(a_1|s_1) \ldots \sum_{s_m,r_m,a_m} \mu(s_m r_m|s_{m-1},a_{m-1}) \Pi(a_m|s_m) \sum_{t=1}^m r_t \\
    &= \sum_{s_1,r_1,a_1} \left( \sum_{\substack{q'\in \SetD_{s_1} : \\  a_1=\tau(D_{\rho}(q'|s_1))} } \vartheta(s_1\epsilon r_1|sq,b)\pi(q_1'|s_1\epsilon)\prod_{i=2}^{|q'|}\vartheta(s_1q_{<i}' 0|s_1q_{<i-1}',q_{i-1}')  \pi(q_i'|s_1q_{<i}')  \right) \\
&\ \ldots\sum_{s_m,r_m,a_m}\left( \sum_{\substack{q'\in \SetD_{s_m} : \\  a_m=\tau(D_{\rho}(q'|s_m))} } \vartheta(s_m\epsilon r_1|s_{m-1}\epsilon ,b)\pi(q_1'|s_m\epsilon) \prod_{i=2}^{|q'|}\vartheta(s_mq_{<i}' 0|s_mq_{<i-1}',q_{i-1}')  \pi(q_i'|s_mq_{<i}')  \right) \\
&\sum_{t=1}^m r_t \\
    &= \E_{\vartheta}^{\pi} \!\left[\sum_{t=1}^m r_t|sq,b\right]\!\! 
    = Q_{\vartheta}^{\pi}(sq,b)
\end{align*}\vspace*{-3ex}

Notice that the horizon is defined externally, i.e. an $m$-horizon return in the external environment corresponds to $m$ decoded external actions, and not the number of internal actions taken.
\end{proof}

\paragraph{Proof of Proposition 3}
\begin{proof} ($\Rightarrow$)

 Let $\cN:=\{x_{1:\infty}:\theta(x_{1:n})\not\to\theta_0\}$ be the set of sequences on which convergence fails.
The consistency assumption implies 
\begin{align*}
    0 &~=~ P_{\theta_0}[\cN]\\
    &~=~ \sum_{x_{1:k}\in\cX^k} P_{\theta_0}[\cN|x_{1:k}] P_{\theta_0}[x_{1:k}] \\
    &~\geq~ P_{\theta_0}[\cN|x_{1:k}] P_{\theta_0}[x_{1:k}].
\end{align*} 
Since $P_{\theta_0}[x_{1:k}]>0$ by assumption, the RHS\ can only be $0$ if $P_{\theta_0}[\cN|x_{1:k}]=0$,
which implies $P_{\theta_0}[x_{1:k}\times\cX^{\infty}\setminus\cN|x_{1:k}]=1$,
hence $\theta(x_{1:k}X_{k+1:n})\to\theta_0$ with $P_{\theta_0}[\cdot|x_{1:k}]$-probability 1.

($\Leftarrow$)

Again let $\cN:=\{x_{1:\infty}:\theta(X_{1:n})\not\to\theta_0\}$ be the set of sequences on which convergence fails.
The consistency assumption implies that for all $x_{1:k}$ for which $P_{\theta_0}[x_{1:k}]>0$, we have $P_{\theta_0}[x_{1:k}\times\cX^{\infty}\setminus\cN|x_{1:k}]=1$, which implies that $P_{\theta_0}[\cN | x_{1:k}]=0$. Therefore, 
\[ 
  P_{\theta_0}[\cN] ~=~ \sum_{x_{1:k}\in\cX^k:P_{\theta_0}[x_{1:k}]>0} P_{\theta_0}[\cN|x_{1:k}] P_{\theta_0}[x_{1:k}] = 0 
\]
which implies $P_{\theta_0}[\cX^{\infty}\setminus\cN]=1$,
hence $\theta(X_{1:n})\to\theta_0$ with $P_{\theta_0}$-probability 1.
\end{proof}

\end{document}